\documentclass{easychair}

\usepackage{doc}

\usepackage{times}
\usepackage{xcolor}
\usepackage{soul}
\usepackage[utf8]{inputenc}
\usepackage[small]{caption}

\usepackage{dsfont}
\usepackage{bbold}
\usepackage{bm}
\usepackage{comment}

\usepackage{graphicx}
\usepackage{amssymb,amsmath}
\usepackage{url}
\usepackage{prettyref}
\usepackage{mathptmx}
\usepackage[ruled,vlined]{algorithm2e}
\usepackage{setspace}
\usepackage{algorithmic}

\newtheorem{proposition}{Proposition}

\newcommand{\hspc}{\mbox{\hspace{0.5em}}}           

\newcommand{\mbf}[1]{ {\mathbf #1} }
\newcommand{\tilmbf}[1]{ \tilde{\mathbf #1}}
\newcommand{\rmin}{\mbox{min}}

\newcommand{\vsp}{\rule{0pt}{1.1em}}
\newcommand{\hvsp}{\rule{0pt}{0.9em}}

\title{  MatSat: a matrix-based differentiable SAT solver
}

\author{
Taisuke Sato\inst{1}
\and
Ryosuke Kojima\inst{2}
}

\institute{
National Institute of Informatics (NII), Tokyo, Japan \\
\email{satou\_taisuke@nii.ac.jp}
\and
Graduate School of Medicine, Kyoto University, Japan \\
\email{kojima.ryosuke.8e@kyoto-u.ac.jp}\\
}

\authorrunning{Sato, Kojima}

\titlerunning{ MatSat: a matrix-based differentiable SAT solver }

\begin{document}

\maketitle

\begin{abstract}
We propose a new approach to  SAT solving which solves SAT problems in
vector spaces as  cost minimization of a  differentiable cost function
${\rm J}^{sat}$.  In our  approach, a solution, satisfying assignment,
of a  SAT problem in $n$  variables is represented by  a binary vector
$\mbf{u} \in  \{0,1\}^n$ such that  ${\rm J}^{sat}(\mbf{u}) =  0$.  We
search for  such $\mbf{u}$  in a vector  space $\mathbb{R}^n$  by cost
minimization, i.e., starting from  an initial $\mbf{u}_0$, we minimize
${\rm  J}^{sat}$  to  zero  while iteratively  updating  $\mbf{u}$  by
Newton's method.
We   implemented   our   approach  as   an   incomplete   matrix-based
differentiable SAT  solver MatSat.  Although existing  main-stream SAT
solvers decide each  bit of a solution assignment one  by one, be they
of conflict driven clause learning  (CDCL) type or of stochastic local
search (SLS) type,  MatSat fundamentally differs from them  in that it
updates all variables  at once and continuously  approaches a solution
in a vector space.
We conducted  experiments to  measure the  scalability of  MatSat with
random 3-SAT problems.   In these experiments, for  example, we showed
that  MatSat implemented  on GPU  can solve  the problem  with $n  = 3
\times  10^5$ variables,  demonstrating  the  feasibility of  hardware
acceleration by GPU for matrix-based solvers like MatSat.
We also  compared MatSat with  nine state-of-the-art CDCL and  SLS SAT
solvers  in terms  of execution  time by  conducting experiments  with
several random and  non-random data sets.  In the case  of easy random
SAT, the performance  of MatSat comes between the SLS  solvers and the
CDCL  solvers  whereas  it  is   ranked  1st  on  the  difficult  one.
On the  other hand, MatSat  showed poor performance on  non-random SAT
problems.   To improve  its poor  performance, we  introduced weighted
variables and clauses and confirmed  the effectiveness of the weighted
version of MatSat on non-random SAT.

\end{abstract}



%
%


\section{Introduction}
\label{sec:intro}

The  Boolean satisfiability  problem (SAT)  lies at  the core  of many
fields  including AI  and extensive  efforts have  been made  to build
powerful SAT solvers.  In this paper, we propose a new approach to SAT
solving  which  solves   SAT  problems  as  cost   minimization  of  a
differentiable  cost function  ${\rm J}^{sat}$  in vector  spaces that
contains    a     piecewise    linear    function     $\rmin_1(x)    =
\rmin(x,1)$\footnote{
This paper was presented at
Pragmatics of SAT:
a workshop of the 24nd International Conference on Theory and Applications of Satisfiability Testing
(PoS 2021).
}\footnote{
Although  $\rmin_1(x)$ is  non-differentiable at  $x=1$, we  apply the
term  ``differentiable''  to  it  in  a broader  sense  as  in  recent
approaches  in neural  network that  combine symbolic  computation and
neural computation.
} as  a  continuous  surrogate   for  disjunction\footnote{
Our use of $\rmin_1(x)$ originates  in Łukasiewicz's real valued logic
where the  truth value  $[A \vee B]$  of a disjunction  $A \vee  B$ is
evaluated as $[A \vee B] = \rmin([A]+[B],1)$.
}.

In our approach, a solution of a SAT problem\footnote{
We use ``SAT problem'' and ``SAT instance'' interchangeably.
}  in $n$  variables  is  represented by  a  0-1  vector $\mbf{u}  \in
\{0,1\}^n$ such that ${\rm J}^{sat}(\mbf{u})  = 0$, i.e., $\mbf{u}$ is
a root of  ${\rm J}^{sat}$.  We search for such  $\mbf{u}$ in a vector
space  $\mathbb{R}^n$ by  cost  minimization, i.e.,  starting from  a
random  $\mbf{u}_0$, we minimize  ${\rm  J}^{sat}$  to zero  while
iteratively updating $\mbf{u}$ using Newton's method.
We  implemented  our approach  as  a  matrix-based differentiable  SAT
solver MatSat.   While existing main-stream  SAT solvers search  for a
solution by determining each bit of a candidate assignment one by one,
be  they  of  conflict  driven  clause  learning  (CDCL)  type  or  of
stochastic local search (SLS)  type, MatSat fundamentally differs from
them  in  that it  updates  all  variables  at once  and  continuously
approaches a solution in a high  dimensional vector space with help of
gradient information.

In a  broader perspective,  MatSat is considered  as an  embodiment of
recently  emerging  differentiable   approaches  which  solve  logical
problems by  converting them to  differentiable forms in  a continuous
space.
We pick up  some of them here  as the background of our  work.  In the
field of  logic programming, Sato  finds that Datalog programs  can be
computed  orders of  magnitude  faster than  the traditional  symbolic
evaluation  by way  of matrix  equations \cite{Sato17b}.   Also it  is
shown that  inventing new  predicates can  be carried  out on  a large
scale  in a  vector space  by minimizing  a loss  function made  up of
matrices representing them \cite{Sato18}.   Sakama et al.\ developed a
matrix-based  approach   to  partial  evaluation  of   logic  programs
\cite{Sakama18}.

Aside  from these  logic-based  approaches,  there are  differentiable
approaches closer to neural networks.
Evans  and  Grefenstette  showed  how  to  incorporate  differentiable
components into an inductive logic programming framework by learning a
weight matrix associated  with candidate clauses for  a target program
to be synthesized \cite{EvansG18}.
Manhaeve  et  al.\  presented  a  neural  extension  of  probabilistic
modeling  language   ProbLog  allowing  ``neural   predicates''  whose
probabilities are computed by neural networks \cite{Manhaeve18}.
Cingillioglu and Russo described a  way of handling logical entailment
in  logic programs  by  their Iterative  Memory  Attention model  that
manipulates a vectorized  normal logic program and  a vectorized query
atom \cite{Cingillioglu18}.

The approaches described  above deal with logic in one  way or another
but  are  not  directly  aimed  at SAT  solving.   There  are  however
researchers aiming at differentiable SAT solvers.  For example,
Nickles incorporated  a differentiable evaluation function  into a SAT
solver to  choose a decision  literal for  sampling of answer  sets in
probabilistic answer set programming \cite{Nickles18b}.
Wang  et.\ al  built  a  MAXSAT solver  based  on  the combination  of
semidefinite  programming  relaxation  and  branch-and-bound  strategy
\cite{Wang19b}.
Selsam  et al.\  proposed a  neural  net classifier  NeuroSAT for  SAT
problems  that  learn  embeddings  of  a  SAT  problem  through  three
perceptrons and  two LSTMs so that  the system predicts one  bit, i.e.,
the satisfiability of the problem \cite{Selsam19}.

Compared  to these  differentiable  SAT solvers,  the architecture  of
MatSat is quite simple.  When viewed as a neural network, it is just a
one-layer  neural  network  applying  a  piecewise  linear  activating
function  and  sum-product operation  to  a  matrix describing  a  SAT
problem.  From  the viewpoint  of existing SAT  solvers, MatSat  is an
incomplete solver  unable to  solve UNSAT problems.   Nonetheless when
applied to SAT problems, it sometimes  outperforms state-of-the-art SAT
solvers      as      far      as     experimental      results      in
Section~\ref{sec:experiments}  are concerned\footnote{

Programs and data used for experiments in this paper are accessible from\\
{\tt https://drive.google.com/drive/folders/1xtqIcpvbW5USqb6CjvXK5zqcZpBqykts?usp=sharing}
}.
It is  also scalable, solving random 3-SAT problems  with up to
$3\times 10^5$ variables when implemented on GPU.

Our contributions thus include a  reformulation of SAT solving as cost
minimization  in vector  spaces with  a logic-based  cost function,  a
proposal  of matricized  differentiable SAT  solver MatSat  suited for
multicore and GPU architecture,  and experimental demonstration of the
strength of MatSat together with its weakness.

\section{Preliminaries}
\label{sec:preliminaries}
In this paper, bold italic capital  letters such as $\bm{A}$ stand for
a real matrix whereas bold italic  lower case letters such as $\bm{a}$
stand  for a  real  vector.   $\| \bm{a}  \|_1  = \sum_i  |\bm{a}(i)|$
denotes  the  1-norm of  a  vector  $\bm{a}$  and  $\| \bm{a}  \|_2  =
\sqrt{\sum_i \bm{a}(i)^2}$  is the  2-norm of  $\bm{a}$.  For  two $n$
dimensional  vectors $\bm{a}$  and $\bm{b}$,  we use  $(\bm{a} \bullet
\bm{b})$ to denote their inner product (dot product) and $\bm{a} \odot
\bm{b}$  to  denote  their  Hadamard  product,  i.e.,  $(\bm{a}  \odot
\bm{b})(i)  =  \bm{a}(i)\bm{b}(i)$ for  $i  (1\leq  i \leq  n)$.   The
dimension $n$ of  a vector $\bm{a}$ is sometimes called  the length of
$\bm{a}$. For a  scalar $\theta$, $\bm{a}_{<\theta}$  stands for a
binary vector  such that $\bm{a}_{<\theta}(i) = 1$  if $\bm{a}(i) < \theta$
and $\bm{a}_{<\theta}(i) = 0$ otherwise.  $\mbf{1}_m$
designates  an all-one  column  vector of  length  $m$.  $\rmin_1(x)  =
\rmin(x,1)$  designates the  lesser of  1 and  $x$.  $\rmin_1(\bm{a})$
means  the  component-wise application  of  $\rmin_1(x)$  to a  vector
$\bm{a}$.

\section{Matricized SAT}
\label{sec:SAT}
We encode a SAT instance in CNF form with $m$ clauses in $n$ variables
into an  $m \times  2n$ binary  matrix.  For example,  we encode  a SAT
instance $S_0 = (a\vee b\vee \bar{c})\wedge(\bar{a}\vee \bar{b})$ into
$\mbf{Q}_{0} \in \mathbb{R}^{2 \times 3}$ below.

\[
\mbf{Q}_{0} =
\begin{array}{c}
\begin{array}{cccccc}
  a & b & c & \bar{a} & \bar{b} & \bar{c} \\
\end{array}
\\
\left[
\begin{array}{cccccc}
  1 & 1 & 0 & 0 & 0 & 1 \\
  0 & 0 & 0 & 1 & 1 & 0 \\
\end{array}
\right]
\end{array}
\]\\

\noindent
The idea  behind $\mbf{Q}_{0}$ is  that when a  SAT instance $S  = C_1
\wedge\cdots\wedge   C_m$   of   $m$    clauses   in   $n$   variables
$\{a_1,\ldots,a_n\}$   is   given,    we   introduce   $2n$   literals
$\{a_1,\ldots,a_n, \neg a_1,\ldots,\neg a_n\}$  in this order together
with a binary matrix $\mbf{Q} \in  \{0,1\}^{m \times 2n}$ and use them
as a column  index for $\mbf{Q}$.  $\mbf{Q}$ represents $S$  in such a
way that $\bm{Q}(i,:)$, the $i$-th row in $\bm{Q}$, encodes the $i$-th
clause $C_i$  ($1 \leq i \leq  m$).  More concretely, when  the $i$-th
clause $C_i$ ($1 \leq i \leq  m$) contains a positive literal $a_j$($1
\leq  j  \leq  n$)  (resp.   negative  literal  $\neg  a_j$),  we  put
$\mbf{Q}(i,j)=1$      (resp.      $\mbf{Q}(i,j+n)=1$).       Otherwise
$\mbf{Q}(i,j)=  \mbf{Q}(i,j+n)  =  0$.   We  call  $\mbf{Q}$  an  {\em
  instance matrix\/} for $S$.
A truth value assignment  (interpretation) for $\{a_1,\ldots,a_n\}$ is
identified with  an {\em assignment  vector\/}, i.e., a  binary column
vector $\mbf{u} \in  \{0,1\}^n$ of length $n$ such  that $\mbf{u}(j) =
1$ (resp.  $0$) holds if-and-only-if the variable $a_j$ is true (resp.
false)  in  the  assignment  ($1  \leq  j  \leq  n$).   Henceforth  we
interchangeably use 1 for true and 0 for false.

Now  introduce the  {\em  dualized assignment  vector\/} $\mbf{u}^d  =
[\mbf{u};\mbf{1}_n-\mbf{u}]$   of  $\mbf{u}$.    It   is  a   vertical
concatenation  of  two column  vectors  $\mbf{u}$  and its  complement
$\mbf{1}_n-\mbf{u}$\footnote{
This  concatenation  of  assignment  vectors  is  first  introduced  in
\cite{Sato20}.
}.  Put $\delta_i  = \bm{Q}(i,:)\mbf{u}^d$.   It is
immediate to see that $\delta_i$ gives  the number of literals true in
$\mbf{u}$  contained   in  $C_i$.   So   if  $\delta_i  \geq   1$,  or
equivalently, if $\rmin_1(\delta_i)=1$, we know  that $C_i$ is true in
$\mbf{u}$.   Otherwise $\rmin_1(\delta_i)=0$  and  $C_i$  is false  in
$\mbf{u}$.   Thus  a  clause  $C_i$   is  evaluated  by  $\mbf{u}$  to
$\rmin_1(\delta_i) = \rmin_1(\bm{Q}(i,:)\mbf{u}^d) \in \{0,1\}$ and if
all    $\rmin_1(\delta_i)$'s   are    1,    in    other   words,    if
$\rmin_1(\bm{Q}\mbf{u}^d) =  \mbf{1}_m$ holds, $S$ is  satisfiable and
$\mbf{u}$ is  a solution,  i.e., satisfying assignment.   We summarize
discussion so far as Proposition~\ref{prop:1}.

\begin{proposition}
Let $\bm{Q} \in \{0,1\}^{m \times 2n}$  be an instance matrix for a SAT
instance  $S =  C_1\wedge\cdots\wedge  C_m$ with  $m$  clauses in  $n$
variables and $\mbf{u}$  an assignment for the variables  in $S$.  $S$
is satisfiable if-and-only-if $\rmin_1(\bm{Q}\mbf{u}^d) = \mbf{1}_m$.
\label{prop:1}
\end{proposition}

\begin{proof}
We prove  the only-if  part.  Suppose  $S$ is  satisfiable and  has a
satisfying assignment $\mbf{u}$.   For every $i$ ($1 \leq  i \leq m$),
$C_i$ has  at least  one literal  true in  $\mbf{u}$.  So  $\delta_i =
\bm{Q}(i,:)\mbf{u}^d \geq 1$  and hence $\rmin_1(\bm{Q}(i,:)\mbf{u}^d)
= 1$ holds for every $i$, which means $\rmin_1(\bm{Q}\mbf{u}^d)$ is an
all-one vector $\mbf{1}_m$.
\end{proof}

Consider for example an assignment $\mbf{u}_0 = [1\; 0\; 0]^T$ for the
variables $\{a, b,  c \}$ in $S_0$.  Then we  have $\mbf{u}_0^d = [1\;
  0\; 0\; 0\; 1\; 1]^T$ and $\delta_1 = \bm{Q}_0(1,:)\mbf{u}_0^d = 2$.
Accordingly $\rmin_1(\delta_1)=1$  and the  first clause  $a\vee b\vee
\bar{c}$ in $S_0$ is true in $\mbf{u}_0$.  Likewise $\rmin_1(\delta_2)
= \rmin_1(\bm{Q}_0(2,:)\mbf{u}_0^d)  = 1$ and the  second clause $\neg
a\vee \neg b$ is true in $\mbf{u}_0$ as well.  So all clauses are true
in  $\mbf{u}_0$,  i.e.,  $\rmin_1(\bm{Q}_0\mbf{u}_0^d)  =  [1\;1]^T  =
\mbf{1}_2$  in notation  and $\mbf{u}_0$  is  a solution  for the  SAT
instance $S_0$.

\section{SAT solving by cost minimization}
\label{sec:cost}
According to  Proposition~\ref{prop:1}, SAT  solving is  equivalent to
finding   an  assignment   vector  $\mbf{u}   \in  \{0,1\}^n$   making
$\rmin_1(\bm{Q}\mbf{u}^d)  =   \mbf{1}_m$  true  where   $\mbf{Q}  \in
\{0,1\}^{m \times  2n}$ is an instance  matrix for $m$ clauses  in $n$
variables.  We  solve this problem  by cost minimization applied  to a
relaxed  equation  $\rmin_1(\bm{Q}{\tilmbf{u}}^d)   =  \mbf{1}_m$  for
$\tilmbf{u}     \in     \mathbb{R}^n$    where     $\tilmbf{u}^d     =
[\tilmbf{u};\mbf{1}_n  -  \tilmbf{u}]$.   We first  introduce  a  cost
function   ${\rm   J}^{sat}:  \tilmbf{u}   \in   \mathbb{R}^{n}\mapsto
\mathbb{R}$ below where $\ell>0$.
\begin{eqnarray}
{\rm J}^{sat}
   & = &  (\mbf{1}_m \bullet (\mbf{1}_m - \rmin_1(\mbf{Q}\tilmbf{u}^d)))
        + (\ell/2)\cdot\|\tilmbf{u}\odot(\mbf{1}_n - \tilmbf{u}) \|_2^2
            \label{eq:jsat}
\end{eqnarray}

\begin{proposition}
Let $\bm{Q} \in \{0,1\}^{m \times 2n}$ be an instance matrix for a SAT
instance $S$.   ${\rm J}^{sat} =  0$ if-and-only-if $\tilmbf{u}$  is a
binary vector $\in \{0,1\}^n$ and $\tilmbf{u}$ represents a satisfying
assignment for $S$.
\label{prop:2}
\end{proposition}

\begin{proof}
Let  $\tilmbf{u}$ be  an assignment  vector satisfying  $S$.  Then  by
Proposition~\ref{prop:1},       we       have       $\mbf{1}_m       =
\rmin_1(\mbf{Q}\tilmbf{u}^d)$.  So  the first term in  ${\rm J}^{sat}$
is zero.   Since $\tilmbf{u}\in \{0,1\}^n$,  the second term  in ${\rm
  J}^{sat}$ is zero  as well. Hence ${\rm J}^{sat} =  0$.  Now suppose
${\rm  J}^{sat}  =  0$.   Since  every  term  in  ${\rm  J}^{sat}$  is
non-negative,    we   have    $(\mbf{1}_m    \bullet   (\mbf{1}_m    -
\rmin_1(\mbf{Q}\tilmbf{u}^d))) = 0$  and $\| \tilmbf{u}\odot(\mbf{1}_n
- \tilmbf{u})  \|_2^2 =  0$.   The latter  implies  $\tilmbf{u}$ is  a
binary   vector  whereas   the  former   implies  that   $\mbf{1}_m  -
\rmin_1(\mbf{Q}\tilmbf{u}^d)=\mbf{0}$,                           i.e.,
$\rmin_1(\mbf{Q}\tilmbf{u}^d)  =  \mbf{1}_m$ holds.   Therefore  every
clause in $S$ is true in $\tilmbf{u}$.
\end{proof}

Proposition~\ref{prop:2} tells  us that  a SAT solution  is a  root of
${\rm J}^{sat}$.   So we compute  it by  Newton's method\footnote{
We tested gradient descent minimization using Adam \cite{Kingma15}
but it turned out to be slower than Newton's method.
}. We  need the
Jacobian  $\mbf{J}_{acb}^{sat}$ of  ${\rm J}^{sat}$.   Split $\mbf{Q}$
into  $\mbf{Q}   =  [\mbf{Q}_1\,  \mbf{Q}_2]$  where   $\mbf{Q}_1  \in
\{0,1\}^{m  \times   n}$  is   a  binary  matrix   holding  occurrence
information about  positive literals  in each clause  while $\mbf{Q}_2
\in   \{0,1\}^{m    \times   n}$    is   the   one    about   negative
literals.  $\mbf{J}_{acb}^{sat}$ is  derived  as follows.

Put $\mbf{c} = \mbf{Q}\tilmbf{u}^d = \mbf{Q}_2\mbf{1}_n + (\mbf{Q}_1-\mbf{Q}_2)\tilmbf{u}$
and $\mbf{d}  = \tilmbf{u} \odot (\mbf{1}_n - \tilmbf{u}) \odot (\mbf{1}_n - 2\tilmbf{u})$.
Let $\tilmbf{u}_p = \tilmbf{u}(p)$ be the  $p$-th component  in  $\tilmbf{u}$   ($1  \leq  p  \leq  n$)  and
$\mbf{I}_p$ a  zero vector  except for the  $p$-th component  which is one.
Note $\frac{ \partial \tilmbf{u} }{ \partial \tilmbf{u}_p } = \mbf{I}_p$.
We  compute   the  partial   derivative  of   ${\rm  J}^{sat}$ w.r.t.\ $\tilmbf{u}_p$

\begin{eqnarray*}
\frac{ \partial {\rm J}^{sat} }{ \partial \tilmbf{u}_p }
  & = & (\mbf{1}_n \,\bullet\, (-(\mbf{c}_{<1}) \odot ((\mbf{Q}_1-\mbf{Q}_2)\mbf{I}_p)) ) + \ell\cdot (\mbf{d} \,\bullet\, \mbf{I}_p)\\
  & = & -( (\mbf{Q}_1-\mbf{Q}_2)^{T}(\mbf{c}_{<1}) \,\bullet\, \mbf{I}_p) + \ell\cdot (\mbf{d} \,\bullet\, \mbf{I}_p)\\
  & = & ( ((\mbf{Q}_2-\mbf{Q}_1)^{T}(\mbf{c}_{<1}) + \ell\cdot\mbf{d}) \,\bullet\, \mbf{I}_p). \\
\end{eqnarray*}

\noindent
Since $p$ is arbitrary, by  substituting original formulas for $\mbf{c}$ and
$\mbf{d}$ respectively, we reach

\begin{eqnarray}
\mbf{J}_{acb}^{sat}
  & = & (\mbf{Q}_2 - \mbf{Q}_1)^T(\mbf{Q}\tilmbf{u}^d)_{<1} \nonumber \\
  &   & \hspace{1em} +\; \ell\cdot(\tilmbf{u} \odot (\mbf{1}_n-\tilmbf{u}) \odot (\mbf{1}_n-2\tilmbf{u})).
            \label{jacob:sat}
\end{eqnarray}

\noindent
We  implement  Newton's method  by  starting  from an  initial  vector
$\tilmbf{u}_{ini}$ and updating $\tilmbf{u}$ by
\begin{eqnarray}
 \tilmbf{u}_{new} \;\leftarrow\;
      \tilmbf{u} - ({\rm J}^{sat}/\|\mbf{J}_{acb}^{sat} \|_F^2)\mbf{J}_{acb}^{sat}
         \label{update:matsat}
\end{eqnarray}
\noindent
until    convergence.      Note    that    the     updating   scheme
(\ref{update:matsat})\footnote{
This  is obtained  from solving  a  linear equation  ${\rm J}^{sat}  +
(\mbf{J}_{acb}^{sat}  \bullet (\tilmbf{u}_{new}  - \tilmbf{u}))  = 0$
w.r.t.\  $\tilmbf{u}_{new}$  which is  derived  from  the first  order
Taylor polynomial of ${\rm J}^{sat}$.
} has  the same form as  gradient descent. The difference  is that the
learning rate $\alpha = {\rm J}^{sat}/\|\mbf{J}_{acb}^{sat} \|_F^2$ is
not a constant and automatically adjusted as convergence proceeds.

To understand what the update scheme (\ref{update:matsat}) does, let's
look  at  the  first  term $\mbf{e}$  in  $\mbf{J}_{acb}^{sat}$  where
$\mbf{e} =  (\mbf{Q}_2 -  \mbf{Q}_1)^T(\mbf{Q}\tilmbf{u}^d)_{<1}$.  We
consider the case in which the second term in ${\rm J}^{sat}$ is small
and $\tilmbf{u}$  is close  to a  binary assignment  vector $\mbf{u}$,
i.e.\ $\tilmbf{u}  \approx \mbf{u}$.  Suppose  so. Let $S_{f}$  be the
set of clauses  in $S$ falsified by $\mbf{u}$. The  members of $S_{f}$
are  detected  as   a  component  being  one  in   the  binary  vector
$(\mbf{Q}\tilmbf{u}^d)_{<1}$.  Assume  that a variable  $v_p$ assigned
$\tilmbf{u}(p)$  occurs positively  $n_1$ times  and negatively  $n_2$
times respectively in the set $S_{f}$.  Then we have $\mbf{e}(p) = n_2
- n_1$.  Since $\mbf{J}_{acb}^{sat}(p) \approx \mbf{e}(p)$\footnote{
The second term in (\ref{jacob:sat}) is small because we assume
$\tilmbf{u}$  is close  to a binary vector.
}, we see that when $v_p$ negatively occurs more often than positively
in $S_{f}$, $\mbf{e}(p)>0$ and the update scheme (\ref{update:matsat})
will decrease $\tilmbf{u}_p$,  making it closer to $0$  (false).  As a
result, every clause  in $S_{f}$ containing $\neg v_p$  gets closer to
true,  which will  increase the  number of  satisfied clauses  in $S$.
Similar understanding is possible in the case of $\mbf{e}(p)<0$.

It is  interesting to  note that  this behavior  of the  update scheme
(\ref{update:matsat})  bears a  strong resemblance  to that  of an SLS
solver which flips a variable  in the currently unsatisfied clauses to
increase satisfied clauses, though our approach is continuous.

\begin{algorithm}[tb]
\caption{\hspace{1em} MatSat algorithm}
\label{alg:sat}
\mbox{\textbf{Input}: an instance matrix $\mbf{Q} \in \{0,1\}^{m \times 2n}$} \\
     \mbox{\hspace{2em} for a SAT instance $S$, integers max\_try and max\_itr} \\
\mbox{\textbf{Output}: an assignment vector $\mbf{u}$ for the variables in $S$ and} \\
      \mbox{\hspace{2em} $error$ = the number of unsatisfying clauses by $\mbf{u}$}
\begin{algorithmic}[1]
   \STATE {initialize $\tilmbf{u} \in \mathbb{R}^n$ by a uniform distribution {U}(0,1)}
   \FOR{$p = 1$ \textbf{to} max\_try}
      \FOR {$q = 1$ \textbf{to} max\_itr}
         \STATE {compute ${\rm J}^{sat}$ by (\ref{eq:jsat}) and $\mbf{J}_{acb}^{sat}$ by (\ref{jacob:sat}) }
         \STATE {update $\tilmbf{u}$ by (\ref{update:matsat})}
         \STATE {threshold $\tilmbf{u}$ to a binary vector $\mbf{u}$}
         \STATE {compute $error = \| \mbf{1}_m - \rmin_1(\mbf{Q}\mbf{u}^d) \|_1 $}
         \IF{$error = 0$} \STATE {exit $p$-loop} \ENDIF
      \ENDFOR
      \STATE {$\tilmbf{u} = (1-\beta)\cdot\tilmbf{u} + \beta\cdot\Delta$\hspace{1em} \% $0 \leq \beta \leq 1$, $\Delta \sim$ {U}(0,1)}
   \ENDFOR
   \STATE \textbf{return} $\mbf{u}$ and $error$
\end{algorithmic}
\end{algorithm}

We  implement our  cost  minimization  approach to  SAT  solving as  a
matrix-based differentiable solver MatSat whose algorithm is depicted in
{\bf Algorithm 1}.
We add some modifications to a basic cost minimization algorithm which
simply iterates  updating $\tilmbf{u}$  until convergence.   The first
one is  a thresholding operation  in line 6.   Theoretically iterating
the $q$-loop  until ${\rm J}^{sat} =0$  is desired but it  takes a long
time. So  we ``jump  to conclusion''  by thresholding  $\tilmbf{u}$ at
$\theta$    to    a    binary    assignment    vector    $\mbf{u}    =
\tilmbf{u}_{\geq\theta}$ so that $error$, the number of falsified clauses
by  $\mbf{u}$,  is  minimum.   Currently  we  apply  grid  search  for
$\theta$.    I.e.,  we   divide   the  interval   [min($\tilmbf{u}$)\;
  max($\tilmbf{u}$)]  into 200 levels  and  choose  the best  $\theta$
giving minimum error\footnote{
When  ${\rm  J}^{sat}(\tilmbf{u})$  becomes   close  to  zero,  so  is
$\|\tilmbf{u}\odot(\mbf{1}_n   -   \tilmbf{u})   \|_2^2$   and   hence
$\tilmbf{u}$ is close  to some binary vector  $\mbf{u}$.  We therefore
may  reasonably expect  from the  continuity of  ${\rm J}^{sat}$  that
${\rm J}^{sat}(\mbf{u})$ is an integer close to zero, or zero itself.
}.

The second one is a retry  with perturbation in line 12.  When $error$
does not  reach zero  within max\_itr iterations  in the  $q$-loop, we
give a  little perturbation to  $\tilmbf{u}$ to avoid a  local minimum
(by adding noise sampled from a uniform distribution to each component
of $\tilmbf{u}$) and  restart the next $q$-loop in  the $p$-loop.  The
$p$-loop  is repeated  at most  max\_try  times.  $\Delta$  is an  $n$
dimensional  vector   whose  component  is  sampled   from  a  uniform
distribution  over   $(0,1)$  and   $\beta$  adjusts  the   degree  of
perturbation.  For  example $\beta  = 1$ is  equivalent to  starting a
completely new $q$-loop and $\beta = 0$ to continuing update.  Usually
we use $\beta = 0.5$.

Time complexity per iteration within  $q$-loop is estimated as $O(mn)$
where $m$ is  the number of clauses and $n$,  the number of variables,
in a  SAT instance as only  multiplication of $m \times  n$ matrix and
vector of length $n$ is used in the algorithm.

\section{Experiments}
\label{sec:experiments}
To examine the viability of our cost minimization approach, we conduct
three experiments with MatSat\footnote{
The first and third experiments are carried out using GNU Octave 4.2.2
and Python  3.6.3 on a  PC with Intel(R) Core(TM) i7-10700@2.90GHz CPU.
}.  The first two are to examine the scalability of MatSat.  The third
one  is  a   comparison  of  MatSat  with   the  state-of-the-art  SAT
solvers\footnote{
\label{datasets}
MatSat  and all  data  sets used  for experiments  in  this paper  are
donwloadable from\\
{\tt https://drive.google.com/drive/folders/1xtqIcpvbW5USqb6CjvXK5zqcZpBqykts?usp=sharing}
}.

\subsection{Scalability}
We examine  how far MatSat works  when the problem size  gets large by
conducting two random 3-SAT experiments.   In the first experiment, we
use  random  SAT  instances   called  ``forced  uniform  random  3-SAT
instance''.  A forced  uniform random 3-SAT instance  with $m$ clauses
in $n$ variables is generated  as follows.  We first randomly generate
a   hidden   assignment   $a_h$   over   $n$   variables.    We   also
equiprobabilistically  generate  a  clause containing  three  mutually
different literals  and if it  is satisfied by  $a_h$, we keep  it. We
repeat this  process $m$ times  and convert  the resulting set  of $m$
clauses to an instance matrix $\mbf{Q} \in \{0,1\}^{m \times 2n}$, the
input for MatSat.

\begin{table}[th]
\caption{Scalability for random 3-SAT}
\label{tbl:3SAT}
\begin{center}
{\small
\begin{tabular}{crrrrrr}            \\[2pt] \hline\hline
\vsp$n$           & $10k$     & $20k$      & $40k$        & $60k$       & $80k$       & $100k$   \\
$m$               & \multicolumn{6}{c}{(\; $n \times 4.26$ \;)}   \\[2pt] \hline
\vsp\#instance    &   5    &   5     &   5     &   5    &  5     &  5     \\[1pt]
time(s)           & 13.0   & 44.8    & 143.8   & 281.0  & 406.2  & 503.7  \\[1pt]
std               &  3.7   & 12.9    &  23.6   & 32.3   & 45.8   & 62.0   \\[1pt] \hline\hline
\end{tabular}
}
\end{center}
\end{table}

While  varying $n$,  we  create  such SAT  instances  containing $m  =
n\times 4.26$ clauses, run MatSat on them with experimental parameters
set to max\_try = 100 and  max\_itr = 1000 and measure execution time.
In more  detail, for  each $n  \in \{10k,20k,40k,60k,80k,100k  \}$, we
generate five  forced uniform random  3-SAT instances containing  $m =
n\times  4.26$ clauses  in $n$  variables  and for  each instance,  we
repeat five trials of  SAT solving.  Table~\ref{tbl:3SAT} presents the
experimental  result.   In Table~\ref{tbl:3SAT},  ``\#instance'',  the
number of  generated instances  for each  $n$, is  fixed to  five.  25
trials are conducted in total for each $n$ and their average execution
time is recorded as ``time''.  We  add that MatSat finds a solution in
all of $25 \times 6 = 130$ SAT solving trials.

To  visualize   time  dependency,  we  plot   average  execution  time
w.r.t.\    $n$   with    std    (sample    standard   deviation)    in
Figure~\ref{fig:exectime}.   It  shows  almost  linear  dependency  of
execution  time  on  $n$  contrary to  the  expected  time  complexity
$O(n^2)$\footnote{
Because  the time  complexity of  update  in the  MatSat algorithm  is
$O(mn)$ and $m = n \times 4.26$ holds in this experiment.
}.   Although   what  causes   this  linearity  remains   unclear,  it
empirically  demonstrates the  scalability of  MatSat for random 3-SAT.

\begin{figure}[htb]
  \centering
  \includegraphics[width=8cm]{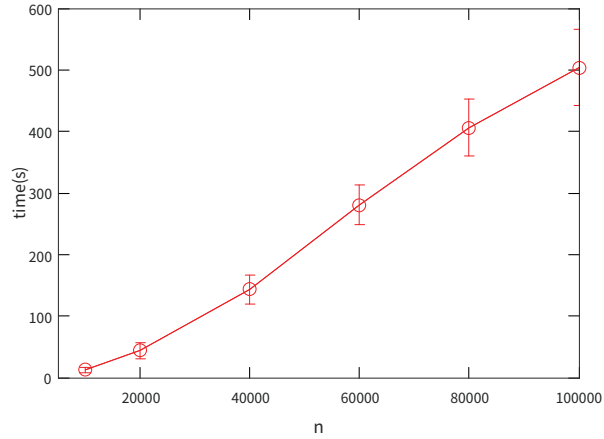}
  \caption{Execution time with random 3-SAT w.r.t. $n$ ($m=4.26n$)}
  \label{fig:exectime}
\end{figure}

We also conduct  the second experiment to test the  affinity of MatSat
with GPU  technology for large SAT  problems with $n$ up  to 300k.  We
reimplement MatSat using GPU (GeForce®  GTX 1080 Ti) as MatSat-GPU and
compare it with Sparrow2Riss-2018 which  is the winner of Random Track
in SAT  competition 2018 \cite{Heule18}  using large scale  random SAT
problems.

\begin{figure}[hbt]
  \centering
  \includegraphics[width=13.5cm]{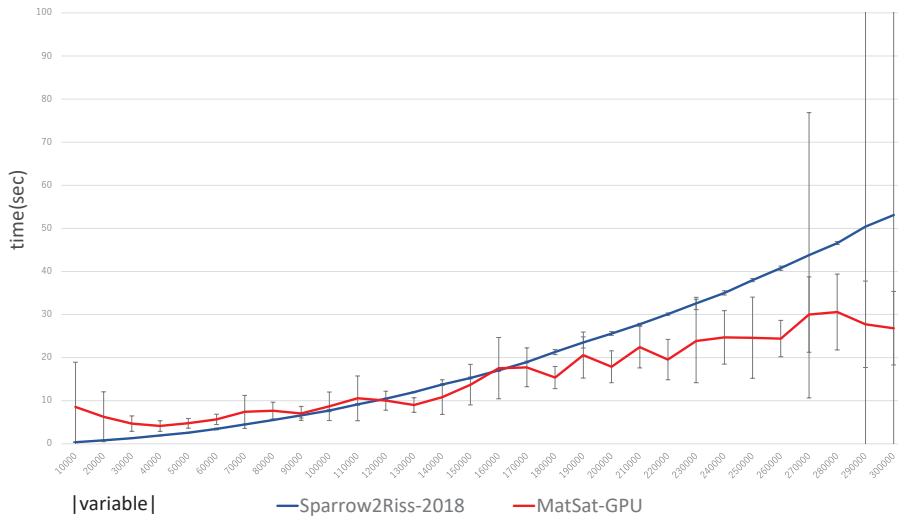}
  \caption{ MatSat-GPU(red) and Sparrow2Riss-2018(blue) up to $n$ = 300k }
  \label{fig:exectimeGPUMatSat}
\end{figure}

In the  experiment, as in  the previous experiment, while  varying $n$
from $n$  = $10k$  to $300k$  with $m =  4.26 \times  n$, we  run each
solver 10 times on random SAT instances generated from different seeds
and measure execution time\footnote{
We set max\_itr  = 1000, max\_try = 100 for  MatSat-GPU.  Given these
parameters  MatSat-GPU did  not terminate  7 times  in the  total $10
\times 30 = 300$ trials of SAT solving.  Default settings are used for
Sparrow2Riss-2018.
}.   Average   execution  time   by   each   solver  is   plotted   in
Figure~\ref{fig:exectimeGPUMatSat}(MatSat-GPU:        red       line,
Sparrow2Riss-2018: blue  line) with  std (sample  standard deviation).
First, we can see that MatSat-GPU runs dozens of times faster than the
CPU version of MatSat, depending  on the number of variables, compared
to the experiment in Figure 1.
We also see that the average execution time by MatSat-GPU slowly rises
w.r.t.  $n$ until $n=300k$ while  that of Sparrow2Riss-2018 rises more
sharply.   As  a  result, although  Sparrow2Riss-2018  initially  runs
faster  than  MatSat-GPU,  their  curves  cross  around  $n=170k$  and
MatSat-GPU runs faster than Sparrow2Riss-2018 thereafter.

What  is  more conspicuous  here  is  the  difference in  std,  sample
standard  deviation.  While  the std  of MatSat-GPU  is rather  stable
through all  $n$, Sparrow2Riss-2018 is  not.  Initially it  has almost
zero std but suddenly  has a large std in $n$ =  270k and this happens
again in $n$ = 290k and $n$  = 300k. The abrupt explosion of std seems
to  imply that  for Sparrow2Riss-2018,  most of  generated random  SAT
instances  are   easily  solvable   but  a  very   few  of   them  are
extraordinarily hard, causing abrupt explosion of std.

This   experiment   demonstrates   the   scalability   of   MatSat-GPU
w.r.t.\ random  3-SAT and at  the same  time the affinity  of MatSat's
matrix-based  architecture  for  GPU.   It  additionally  reveals  the
relative stability of  the behavior of MatSat in  SAT solving compared
to Sparrow2Riss-2018 when dealing with large random 3-SAT.

\subsection{Comparing MatSat with modern SAT solvers}

Here  we compare  MatSat with  the state-of-the-art  SAT solvers.
We consider four solvers from SAT competition 2018 \cite{Heule18,Heule19}\footnote{
{\tt https://helda.helsinki.fi/handle/10138/237063}
}, i.e.\
Sparrow2Riss-2018 (Random Track, winner),
gluHack (Random Track, ranked 2nd),
glucose-3.0\_PADC\_10\_NoDRUP (Random Track,  ranked 3rd) and
MapleLCMDistChronoBT (Main Track winner).
We also consider  MapleLCMDistChrBt-DL-v3 which is
the winner of SAT RACE 2019\footnote{
{\tt http://sat-race-2019.ciirc.cvut.cz}
}
and a basic solver
MiniSat 2.2 \cite{Een03}\footnote{
{\tt http://minisat.se/MiniSat.html}
}.
They are  all CDCL  solvers except  for Sparrow2Riss-2018  which is  a
hybrid of SLS solver (Sparrow) and CDCL solver (Riss) \cite{BM2014}.
We also compare MatSat with three SLS solvers.
They are
probSAT \cite{Balint12}\footnote{
{\tt https://github.com/adrianopolus/probSAT}
} which selects a varible based on a probability distribution,
YalSAT \cite{Balint14}\footnote{
{\tt yalsat-03s.zip} in {\tt http://fmv.jku.at/yalsat/}
}
which  is based  on probSAT  and the  winner of  2017 SAT  competition
Random Track
and
CCAnr 1.1 \cite{Cai15}\footnote{
{\tt https://lcs.ios.ac.cn/~caisw/SAT.html}
}
aiming  at  non-random  SAT   based  on  Configuration  Checking  with
Aspiration (CCA) heuristics.

First we compare MatSat with these nine solvers by random SAT.  We use
three data  sets: Set-A, Set-B  and Set-C.  Set-A contains  500 forced
uniform random 3-SAT instances.  Each instance is composed of $m=2130$
clauses in $n=500$ variables.
Set-B\footnote{
originally taken from {\tt /cnf/rnd-barthel} in {\tt http://sat2018.forsyte.tuwien.ac.at} {\tt /benchmarks/Random.zip}.
It is donwloadable from\\
{\tt https://drive.google.com/drive/folders/1xtqIcpvbW5USqb6CjvXK5zqcZpBqykts?usp=sharing}
}   is   a  random   benchmark   set   taken  from   SAT   competition
2018 \cite{Heule18,Heule19}.  It consists  of 55 random 3-SAT  instances each of
which contains $m$  clauses in $n$ variables where  $n$ = 200$\sim$400
and $m$ = 4.3n.
Set-C is  a small  set containing 10  random 5-SAT  instances selected
from a benchmark set\footnote{
originally taken from {\tt /cnf/Balint} in {\tt http://sat2018.forsyte.tuwien.ac.at} {\tt /benchmarks/Random.zip}.
It is donwloadable from\\
{\tt https://drive.google.com/drive/folders/1xtqIcpvbW5USqb6CjvXK5zqcZpBqykts?usp=sharing}
} submitted for SAT  competition 2018 \cite{Heule18,Heule19}.  Each instance is
generated as ``q-hidden formula'' \cite{Jia05} with 5279 clauses in 250
variables and intended to be hard to SLS solvers\footnote{
Jia et al.\ introduced ``q-hidden formula'' to generate a random k-SAT
satisfiable instance  where the  probability that  a literal  which is
made true by a hidden assignment  occurs positively in the instance is
equal to  the probability that  the same literal occurs  negatively in
the instance \cite{Jia05}.
}.

Experimental settings are as follows.   To run MatSat, we use max\_itr
= 500 and max\_try  = 100 for Set-A and Set-B and  max\_itr = 5000 and
max\_try =  2000 for Set-C. Set-A  is divided into five  sets of equal
size and we  run MatSat on each  of five divided sets  to obtain total
execution time  for 500  random 3-SAT instances.   MatSat successfully
returns a solution in all  cases.  For other solvers, default settings
are used.  timeout is set to 5000s and when it occurs it is counted as
5000s execution time.  We measure  average execution time per instance
for Set-A and Set-C and the total execution time for Set-B.

\begin{table}[hbt]
\caption{Execution time for random SAT}
\label{tbl:sat_time_random}
\begin{center}
\begin{tabular}{crrrc}                          \\ \hline\hline
\vsp    Solver                       & \multicolumn{4}{c}{ Data set } \\  \hline
\vsp                                 & Set-A       & Set-B           & \multicolumn{2}{c}{Set-C} \\
\cline{4-5}
\hvsp                                & time(s)*    &   time(s)       &  time(s)*        &  timeout/10 \\ \hline
\hvsp MatSat                         & 0.0679      &\hspc  18.8      &\hspc {\bf 697.7} & {\bf 0/10}  \\
\hvsp Sparrow2Riss-2018              & {\bf 0.0013}&\hspc   2.3      &\hspc 1544.1      &  3/10 \\
\hvsp gluHack                        & 34.5        &\hspc 537.4      &\hspc 5000.0      & 10/10 \\
\hvsp glucose-3.0\_PADC\_10\_NoDRUP  & 42.9        &\hspc 962.4      &\hspc 5000.0      & 10/10 \\
\hvsp MapleLCMDistChronoBT           & 0.34        &\hspc 519.1      &\hspc 5000.0      & 10/10 \\
\hvsp MapleLCMDistChronoBT\_DL\_v3   & 1.96        &\hspc 762.3      &\hspc 5000.0      & 10/10 \\
\hvsp MiniSat 2.2                    & 3.28        &\hspc  50.1      &\hspc 5000.0      & 10/10 \\
\hvsp probSAT                        & 0.0050      &\hspc  0.39      &\hspc 2134.8      & 2/10 \\
\hvsp YalSAT                         & 0.0058      &\hspc {\bf 0.34} &\hspc 1018.6      & 1/10 \\
\hvsp CCAnr 1.1                      & 0.0051      &\hspc  0.48      &\hspc  937.6      & {\bf 0/10} \\ \hline\hline
\end{tabular}
\end{center}
\end{table}

The  experimental result  is shown  in Table~\ref{tbl:sat_time_random}
with best figures in bold  face.  Here time(s)* denotes execution time
per instance. Set-C has the timeout column which denotes the number of
occurrences of timeout in solving Set-C.   We first observe that as far
as Set-A and Set-B are concerned, in terms of execution time, pure CDCL
solvers perform poorly\footnote{
Sparrow2Riss-2018 is not a pure SLS  or CDCL solver.  It first runs as
an SLS solver  and when a predefined amount of  $5\cdot10^8$ flips are
used up, it runs as a CDCL solver \cite{BM2014}.
} while  pure  SLS   solvers  perform  show  their   strength  and  the
performance  of  MatSat   is  just  between  the   two  groups.

More  interesting is  the  case of  Set-C which  is  generated by  the
q-hidden method \cite{Jia05} so that  the occurrence of literals in an
instance is  balanced, and hence  more difficult than random  SAT with
unbalanced occurrences of literals.
In the  case of  Set-C, all  conventional solvers,  CDCL or  SLS type,
cause  timeout except  for  CCAnr 1.1.   Even Sparrow2Riss-2018  which
performs best  for Set-A causes  timeout three times.   In particular,
all pure  CDCL solvers cause  timeout for every  instance. Consequently
MatSat is  ranked 1st on execution  time of Set-C.  We  point out that
MatSat  outperforms  all  pure  CDCL solvers.   This  observations  is
consistently  understandable  if  we  notice that  MatSat  bears  some
similarity to an SLS solver as explained in Section~\ref{sec:cost} and
usually an SLS solver is good at solving random SAT.\\

To  understand  more  deeply  how MatSat  can  outperform  these  CDCL
solvers, we examine the detail  of a typical distribution of execution
time in 100 forced random uniform 3-SAT instances\footnote{
Each instance is a set of $m=2130$ clauses in $n=500$ variables
(the same condition on the Set-B in Table~\ref{tbl:sat_time_random}).
} by MatSat  and three CDCL solvers  (MiniSat 2.2, MapleLCMDistChronoBT
and    MapleLCMDistChronoBT\_DL\_v3)    which     is    depicted    in
Figure~\ref{fig:distrtime}.

Observe       that        MatSat,       MapleLCMDistChronoBT       and
MapleLCMDistChronoBT\_DL\_v3 consume  very similar time  for instances
numbered from about 30 to  100 but concerning the remaining instances,
the latter  two take much  longer time  than MatSat.  In  other words,
about 30\% of the 100 random  SAT instances are extremely hard for the
CDCL solvers  while those  instances are  relatively easy  for MatSat,
which  explains  the  superiority  of average  performance  by  MatSat
compared  to  them.  Also  observe  that  the same  phenomenon  occurs
between MatSat and MiniSat 2.2 where MiniSat 2.2 suddenly takes longer
time  for instances  numbered less  than 5.   So we  may say  that the
execution time  by MatSat  is more  stable and  less sensitive  to the
difference in  individual instances  compared to other  solvers.  This
stability could  be ascribed to  the global nature of  MatSat's update
scheme (\ref{update:matsat})  which always  takes all components  in a
candidate  assignment  into  account unlike  the  sequential  discrete
update  (with almost  inevitable  backtracking) on  components in  the
candidate assignment made by a CDCL solver.\\

\begin{figure}[tbh]
  \centering
  \includegraphics[width=8.5cm]{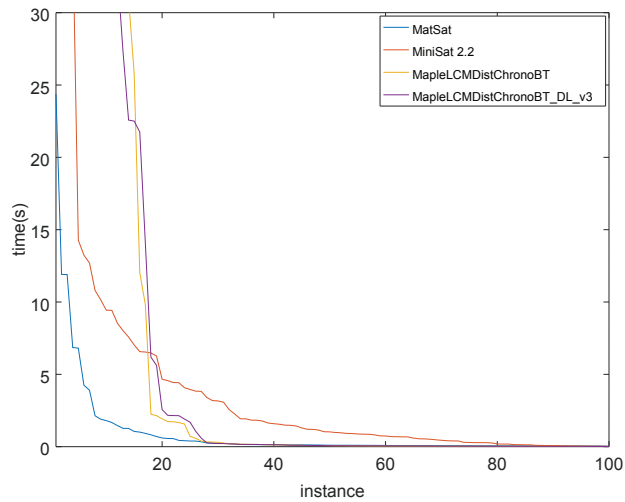}
  \caption{Distribution of execution time for 100 instances (n=500,m=2130)}
  \label{fig:distrtime}
\end{figure}

So far we have been focusing solely on random SAT instances which have
different characteristics  from non-random  ones.  So  we additionally
conduct  an   experiment  to  compare  ten   solvers  by  non-random
(structured)  SAT.

We prepare three benchmarks. The first one is Set-D\footnote{
{\tt SATLIB\_Flat\_Graph\_Colouring/flat30-60} in
{\tt https://www.cs.ubc.ca/{\textasciitilde}hoos/SATLIB/benchm.html}
} taken from  SATLIB \cite{Hoos00}.  It  contains 100  satisfiable 3-SAT
instances where each instance has $m=300$ clauses in $n=90$ variables.
The second one is Set-E\footnote{
{\tt SATLIB\_Morphed\_Graph\_Colouring/sw100-8-lp0-c5} in
{\tt https://www.cs.ubc.ca/{\textasciitilde}hoos/SATLIB/benchm.html}
} also  taken from  SATLIB \cite{Hoos00}.  It contains  100 satisfiable
5-SAT instances  where each instance  has $m=3100$ clauses  in $n=500$
variables.
The third one is Set-F  from  SAT  competition  2018   Main  Track \cite{Heule18,Heule19}\footnote{
\label{footnote_Set_F}
originally taken from {\tt Chen} in {\tt http://sat2018.forsyte.tuwien.ac.at}{\tt /benchmarks/Main.zip}.
It is donwloadable from\\
{\tt https://drive.google.com/drive/folders/1xtqIcpvbW5USqb6CjvXK5zqcZpBqykts?usp=sharing}.
}  submitted   by  Chen   \cite{Chen18},  which   encodes  relativized
pigeonhole   principle  (RPHP)   \cite{Elffers18}   and  contains   20
satisfiable 3-SAT  instances where each instance  has $m=1120$ clauses
in $n=320$ variables (concerning Set-F, it was initially assumed to be
non-random  SAT but  later we  are informed  by reviewers  that Chen's
benchmarks are not non-random and Set-F  should be counted as a random
SAT problem.  We are sorry that we were not aware of the random nature
of Set-F.   So the result  with Set-F should  be taken as  showing the
behavior of MatSat with random SAT like problems).

We run ten solvers, i.e.,
MatSat,
Sparrow2Riss-2018,
gluHack,
glucose-3.0\_PADC\_10\_NoDRUP,
MapleLCMDistChronoBT,
MapleLCMDistChronoBT\_DL\_v3,
MiniSat 2.2,
probSAT,
YalSAT
and
CCAnr 1.1.
Their experimental  settings are  as follows.  To  run MatSat,  we use
max\_itr = 500 and  max\_try = 100 for Set-D,
max\_itr = 2000 and  max\_try = 100 for Set-E and
max\_itr = 500 and  max\_try =  1000 for Set-F.
Given these  parameters, MatSat
successfully finds a  solution for all instances.   For other solvers,
default settings  are used.   timeout is set  to 5000s.   When timeout
occurs,  we equate  execution time  to  5000s.  We  measure the  total
execution time  for Set-D  and Set-E, and  average execution  time per
instance for Set-F.

\begin{table}[hbt]
\caption{Execution time for non-random and random SAT}
\label{tbl:sat_time_non_random}
\begin{center}
\begin{tabular}{crrrc}                          \\ \hline\hline
\vsp    Solver                       & \multicolumn{4}{c}{ Data set } \\  \hline
\vsp                                 & Set-D       & Set-E           & \multicolumn{2}{c}{Set-F} \\
\cline{4-5}
\hvsp                                & time(s)     &  time(s)        & time(s)*       &  timeout/20 \\ \hline
\hvsp MatSat                         & 9.29        &\hspc 679.7      &\hspc {\bf 2.0} & {\bf 0/20}  \\
\hvsp Sparrow2Riss-2018              & 0.23        &\hspc 0.26       &\hspc 251.9     &  1/20       \\
\hvsp gluHack                        & 0.32        &\hspc 0.58       &\hspc 5000.0    & 20/20       \\
\hvsp glucose-3.0\_PADC\_10\_NoDRUP  & 0.45        &\hspc 0.63       &\hspc 4346.2    & 17/20       \\
\hvsp MapleLCMDistChronoBT           & 0.22        &\hspc 0.42       &\hspc 3643.7    & 12/20       \\
\hvsp MapleLCMDistChronoBT\_DL\_v3   & 0.24        &\hspc 0.42       &\hspc 3205.7    & 10/20       \\
\hvsp MiniSat 2.2                    & {\bf 0.21}  &\hspc {\bf 0.21} &\hspc 4345.8    & 15/20       \\
\hvsp probSAT                        & 0.46        &\hspc  0.46      &\hspc  14.3     & {\bf 0/20}  \\
\hvsp YalSAT                         & 0.39        &\hspc  0.45      &\hspc  41.3     & {\bf 0/20}  \\
\hvsp CCAnr1.1                       & 0.39        &\hspc  0.50      &\hspc  14.2     & {\bf 0/20}  \\ \hline\hline
\end{tabular}
\end{center}
\end{table}

Table~\ref{tbl:sat_time_non_random} shows the result with best figures
in  bold face.   time(s)*  in  the table  denotes  execution time  per
instance.  Set-F  has the timeout  column which denotes the  number of
occurrences of timeout in solving 20 instances in Set-F.

MatSat  performs  very poorly  on  Set-D  and Set-E
compared  to other  solvers.  However,  Set-F shows  quite different
performance
(as mentioned before, Set-F should be considered as a random SAT problem).
As   the   timeout   column   in
Table~\ref{tbl:sat_time_non_random}  tells,   MatSat  and   three  SLS
solvers  (probSAT,   YalSAT,  CCAnr   1.1)  causes  no   timeout,  but
Sparrow2Riss-2018  causes timeout  once  and other  four CDCL  solvers
cause timeout ten  times or more.  Also while none  of the SLS solvers
causes timeout, they take much longer time than MatSat to solve Set-F.
As a reuslt,  MatSat outperforms all nine solvers  that participate in
the experiment.

Although we  cannot stretch experimental  results in this  section too
far, we  may say that MatSat  behaves differently in SAT  solving from
existing SAT  solvers, sometimes  giving an advantage  in the  case of
random SAT.

\section{Weighted variables and clauses}
The  experimental  results  in   the  previous  section  uncover  poor
performance of MatSat on non-random  SAT. We here propose to introduce
weighted variables and clauses to  improve the situation.  The idea is
to introduce a real number, {\em variable weight\/}, for each variable
to tell the  solver which variable is important.  Also  we introduce a
real number, {\em clause weight\/},  for each clause.  Those variables
and clauses that have large weights  have the priority to be optimized
by the  solver just like  unit clauses  are specially treated  by unit
propagation.

More  formally, when  an  instance matrix  $\mbf{Q}(m  \times 2n)$  is
given, we introduce a variable  weight vector $\mbf{w}_v (1 \times n)$
and a clause weight vector $\mbf{w}_c (m \times 1)$ w.r.t.\ $\mbf{Q}$.
Let   $a_i$  $(1\leq   i  \leq   n)$  be   the  $i$-th   variable  and
$\mbf{w}_v^*(i)$, the  $i$-th element of $\mbf{w}_v^*$,  be the number
of  total occurrences  of $a_i$  in  the SAT  instance represented  by
$\mbf{Q}$.  We define  the variable weight $\mbf{w}_v  (i)$ for $a_i$
as $\mbf{w}_v (i) = \mbf{w}_v^*(i)/{\rm ave}(\mbf{w}_v^*)$ where ${\rm
  ave}(\mbf{w}_v^*)$ is the average  of elements in $\mbf{w}_v^*$.  We
next define a clause weight $\mbf{w}_c(j)$ $(1\leq j \leq m)$ for the
$j$-th clause $C_j$ represented by the  $j$-th row of $\mbf{Q}$ as the
sum  of  weights of  variables  in  $C_j$.   Finally we  define  ${\rm
  J}^{sat\_w}$, a weighted version of ${\rm J}^{sat}$, by

\begin{eqnarray}
{\rm J}^{sat\_w}
   & = &  (\mbf{1}_m \bullet  (\mbf{w}_c \odot (\mbf{1}_m - \rmin_1(\mbf{Q}\tilmbf{u}^d))))
        + (\ell/2)\cdot\|\mbf{w}_v^T \odot \tilmbf{u}\odot(\mbf{1}_n - \tilmbf{u}) \|_2^2.    \label{eq:jsatw}
\end{eqnarray}
Since $(\mbf{1}_m  - \rmin_1(\mbf{Q}\tilmbf{u}^d))(j)$ is  the falsity
of  the  $j$-th clause  $C_j$,  being  multiplied  by a  large  weight
$\mbf{w}_c(j)$ in  ${\rm J}^{sat\_w}$ means the  continuous falsity of
$C_j$ is more strongly minimized  than other clauses in learning.  Put
$\mbf{Q}    =     [\mbf{Q}_1\,    \mbf{Q}_2]$     where    $\mbf{Q}_1$
(resp. $\mbf{Q}_2$)  represents the positive literals  (resp. negative
literals)   of   $\mbf{Q}$.    Then  ${\rm   J}^{sat\_w}$'s   Jacobian
$\mbf{J}_{acb}^{sat\_w}$ is computed as follows (derivation omitted).
\begin{eqnarray}
\mbf{J}_{acb}^{sat\_w}
   & = &  (\mbf{Q}_2 - \mbf{Q}_1)^T(\mbf{w}_c  \odot [(\mbf{Q}\tilmbf{u}^d) \leq \mbf{1}_m])
        + \ell\cdot(\mbf{w}_v^T \odot \mbf{w}_v^T \odot \tilmbf{u}
              \odot (\mbf{1}_n - \tilmbf{u}) \odot (\mbf{1}_n - 2\tilmbf{u}) )                \label{eq:jacbw}
\end{eqnarray}

We conduct a small experiment on non-random SAT using this weighted version of MatSat.
It minimizes ${\rm J}^{sat\_w}$ (\ref{eq:jsatw}) by Newton's method using (\ref{eq:jacbw}).
In addition to Set-E and Set-F, we use another test set Set-G
from SATLIB \cite{Hoos00} in the experiment\footnote{
It is taken from SATLIB\_MorphGraph/sw100-8-lp1-c5
{\tt (https://www.cs.ubc.ca/{\textasciitilde}hoos/SATLIB/benchm.html)}.
} containing 100 instances of 5-SAT with $m=3100$ clauses in $n=500$ variables.
The experimental result is summarized in Table~\ref{tbl:weight_non_random}.

\begin{table}[hbt]
\caption{Learning time(s) by the weighted and non-weighted MatSat on non-random SAT instances}
\label{tbl:weight_non_random}
\begin{center}
\begin{tabular}{cccc}                          \\ \hline\hline
\vsp   MatSat               &    Set-D    &    Set-E    &  Set-G     \\ \hline
\vsp   non-weighted         &     9.29    &    679.7    &  1277.3    \\
\vsp  (max\_itr max\_try)   &  (500 100)  &  (2k 100)   &  (5k 100)  \\ \hline
\vsp   weighted             &   2.18      &     82.4    &   123.7    \\
\vsp  (max\_itr max\_try)   &  (100,300)  &   (1k,100)  &  (200,100) \\ \hline\hline
\end{tabular}
\end{center}
\end{table}

Comparing the non-weighted version and the weighted version, we notice
that  the weighted  MatSat  learns  much faster,  almost  an order  of
magnitude faster for Set-G, than the non-weighted MatSat, with a small
number of iterations.   Although this experiment is  small, it attests
the  effectiveness of  weighted variables  and clauses  for MatSat  to
control  the learning  process and  suggests  one way  to improve  the
performance of MatSat on non-random SAT.

\section{Related Work}
\label{sec:related_work}
SAT  is a  central  discipline  of computer  science  and the  primary
purpose of this paper is to  bring a differentiable approach into this
established field from the viewpoint of integrating symbolic reasoning
and  numerical  computation.  But,  first,  we put  our  work  in  the
traditional context.

According  to  a comprehensive  survey  \cite{Gu96},  our proposal  is
categorized as continuous-unconstrained  global optimization approach.
We note that a variety of global optimization approaches are described
in \cite{Gu96}  but all encode  clauses as  a product of  some literal
functions as in \cite{Gu96b}. As  a result, the objective function for
$n$ variables becomes a multivariate  polynomial of degree possibly up
to $2n$.  Contrastingly in our approach, clauses are not functions but
constant vectors collectively  encoding a SAT instance  as an instance
matrix,  and far  more importantly,  our objective  function is  not a
simple multivariate  polynomial but a  sum of two terms  with distinct
roles as  explained next.  First  recall our objective  function ${\rm
  J}^{sat}$  is of  the form:  ${\rm  J}^{sat} =  f  + r$  where $f  =
(\mbf{1}_m \bullet (\mbf{1}_m  - \rmin_1(\mbf{Q}\tilmbf{u}^d)))$ using
$\rmin_1(x)$\footnote{
As mentioned  before, the use of $\rmin_1(x)$ comes from logic and
mimics the  evaluation of disjunction.
}and  $r  =   (\ell/2)\cdot\|\tilmbf{u}\odot(\mbf{1}_n  -  \tilmbf{u})
\|_2^2$\footnote{
%
It is possible to replace $l_2$ norm by $l_1$ norm but a tentative
implementation lead to non-smooth convergence.
}.  While  $r$ is just  a multivariate polynomial of  constant degree,
four, working as a penalty term to force 0-1 vectors, $f$ is essential
and navigates  SAT solution  search as a  continuous surrogate  of the
number of false clauses.  Unlike various UniSAT models in \cite{Gu96},
it  is not  a  product of  functions  but a  sum  of piecewise  linear
functions of  $\tilmbf{u}$, which  makes all  the difference  from the
viewpoint  of optimization.   In  particular, when  all components  of
$\mbf{Q}\tilmbf{u}^d$   are  less   than   one,   it  becomes   linear
w.r.t.\ $\tilmbf{u}$.  So when $\tilmbf{u}$ is distant from a solution
and $\mbf{Q}\tilmbf{u}^d$  which is a continuous  approximation to the
number of {\em  true literal\/}s in clauses is  small, the convergence
of $f$ is likely to be fast due to the linearity of $f$.

Now we  focus on related  works from  the viewpoint of  of integrating
symbolic   reasoning   and   numerical  computation.    As   explained
Section~\ref{sec:intro}, there has been a steady stream of tensorizing
logic  programming  \cite{Sato17b,Sato18,Sakama18,Kojima19,Sato20}  in
which for  example, Sato  experimentally implemented  a differentiable
SAT solver in Octave and compared its performance with one CDCL solver
\cite{Sato20}.

There  is also  a trend  of  combining neural  computation with  logic
programming.   For  example,  DeepProbLog  allows the  user  to  write
``neural predicates'' computed by a neural network and the probability
of a  query is computed just  like usual ProbLog programs  while using
probabilities  computed by  the  neural predicates  \cite{Manhaeve18}.
T-PRISM replaces the basic data structure with tensor computation in a
logic-based probabilistic  modeling language PRISM  \cite{Sato01g} and
makes it possible to learn neural networks \cite{Kojima19}.

Direct attempts at  a differentiable SAT solver include  the work done
by Nickles  that integrates the derivative  of a loss function  with a
traditional SAT  solver \cite{Nickles18b}.   Wang shows that  a MAXSAT
can be  incorporated into  a convolutional  neural network  to realize
end-to-end learning of visual  Sudoku problems \cite{Wang19a}.  Selsam
et al.'s  neural net classifier NeuroSAT  achieves 85\% classification
accuracy on moderate size problems with 40 variables and could extract
solutions 70\% of satisfiable cases \cite{Selsam19}.

Last  but not  least, there  are some  works on  the use  of GPU,  for
example,  for  simplification  \cite{Osama19}  and  for  CDCL  solvers
\cite{Osama21}.   In the  latter  case,  Osama et  al.\  used GPU  for
inprocessing such as garbage collection and simplification.

\section{Conclusion}
\label{sec:conclusion}
We proposed  a simple differentiable SAT  solver MatSat.  It is  a new
type of SAT solver, not a CDCL solver or an SLS solver. It carries out
SAT solving in  vector spaces as cost minimization of  a cost function
${\rm   J}^{sat}(\tilmbf{u})$   w.r.t.\    a   continuous   assignment
$\tilmbf{u}$ applied to  a binary matrix $\mbf{Q}$  representing a SAT
instance.   We empirically  confirmed the  scalability of  MatSat.  In
particular, it  is shown that  when implemented  by GPU, it  can solve
random 3-SAT  instances of  various sizes  up to $n  = 3  \times 10^5$
variables.   We then  compared  MatSat with  nine  modern SAT  solvers
including representative  CDCL and SLS  solvers in terms  of execution
time.  When  applied to three  data sets  of random SAT  instances, it
outperforms all  CDCL solvers but  is outperformed by all  SLS solvers
for two easy sets. On the  other hand concerning the data set designed
to be hard for SLS solvers,  it outperforms all nine solvers.  We also
conducted  a  comparative  experiment  with non-random  SAT  sets  and
discovered that MatSat performs poorly  on these sets.  Apparently the
number of experiments is not enough  but these results seem to suggest
the unique behavior of MatSat different from SLS and CDCL solvers.

As future work, we need to conduct more experiments with various types
of SAT  instances including (unforced) uniform  random k-SAT instances
and structured ones.  It is our plan to improve GPU implementation and
introduce heuristics for parameters $\ell$ in (\ref{eq:jsat}) and $\beta$ in {\bf
  Algorithm 1}.  We also plan  to reconstruct MatSat for MAXSAT, which
seems straightforward because $\mbf{v} = \rmin_1(\mbf{Q}\mbf{u}^d)$ in
${\rm J}^{sat}$ computed by MatSat  is actually an indicator vector of
clauses satisfied  by an assignment  vector $\mbf{u}$ and  $\| \mbf{v}
\|_1$ is the number of satisfied clauses.  Hence by relaxing $\mbf{u}$
to $\tilmbf{u}  \in \mathbb{R}^n$ and concurrently  relaxing $\mbf{v}$
to  $\tilmbf{v}\in  \mathbb{R}^n$,  we  would  have  a  differentiable
formulation of MAXSAT that optimizes $\| \tilmbf{v} \|_1$.

MatSat  is at  the cross-roads  of SAT  solving, machine  learning and
(propositional) logic and also closely related to neural networks.  It
is an instance of differentiable approaches emerging in various fields
ranging from neural networks to  logic programming that solve symbolic
problems  in a  differentiable framework  in a  continuous space.   We
expect that  MatSat opens  a new approach  to building  a powerful
GPU-based SAT solver.

\section*{Acknowledgments}
We thank  Professor Katsumi  Inoue and  Koji Watanabe  for interesting
discussions and  useful information  on experimental data.   This work
was supported by the New  Energy and Industrial Technology Development
Organization   (NEDO)  and   JSPS   KAKENHI   Grant  No.17H00763   and
No.21H04905.



\end{document}